\documentclass[letterpaper]{article} 
\usepackage{aaai2026}  
\usepackage{times}  
\usepackage{helvet}  
\usepackage{courier}  
\usepackage[hyphens]{url}  
\usepackage{graphicx} 
\usepackage{natbib}  
\usepackage{caption} 
\usepackage{algorithm}
\usepackage{algorithmic}

\usepackage{newfloat}
\usepackage{listings}

\usepackage{amsmath}
\usepackage{amssymb}
\usepackage{multirow}
\usepackage{amsthm}

\theoremstyle{plain}
\newtheorem{theorem}{Theorem}

\newtheorem{proposition}[theorem]{Proposition}
\theoremstyle{definition}

\theoremstyle{remark}

\DeclareCaptionStyle{ruled}{labelfont=normalfont,labelsep=colon,strut=off} 
\floatstyle{ruled}
\newfloat{listing}{tb}{lst}{}
\floatname{listing}{Listing}
\title{Universal Adversarial Purification with DDIM Metric Loss for Stable Diffusion}
\author{
    Li Zheng\textsuperscript{\rm 1},
    Liangbin Xie\textsuperscript{\rm 1 \rm 2},
    Jiantao Zhou\textsuperscript{\rm 1}\thanks{Corresponding author},
    He YiMin\textsuperscript{\rm 1}
}
\affiliations{
    \textsuperscript{\rm 1}University of Macau\\
    \textsuperscript{\rm 2}Shenzhen Institute of Advanced Technology \\
    yc27908@um.edu.mo,
    lb.xie@siat.ac.cn,
    jtzhou@um.edu.mo,
    mc45058@um.edu.mo
    
}

\usepackage{bibentry}

\begin{document}

\maketitle

\begin{abstract}
Stable Diffusion (SD) often produces degraded outputs when the training dataset contains adversarial noise. Adversarial purification offers a promising solution by removing adversarial noise from contaminated data. However, existing purification methods are primarily designed for classification tasks and fail to address SD-specific adversarial strategies, such as attacks targeting the VAE encoder, UNet denoiser, or both.
To address the gap in SD security, we propose Universal Diffusion Adversarial Purification (UDAP), a novel framework tailored for defending adversarial attacks targeting SD models. UDAP leverages the distinct reconstruction behaviors of clean and adversarial images during Denoising Diffusion Implicit Models (DDIM) inversion to optimize the purification process. By minimizing the DDIM metric loss, UDAP can effectively remove adversarial noise. Additionally, we introduce a dynamic epoch adjustment strategy that adapts optimization iterations based on reconstruction errors, significantly improving efficiency without sacrificing purification quality.
Experiments demonstrate UDAP’s robustness against diverse adversarial methods, including PID (VAE-targeted), Anti-DreamBooth (UNet-targeted), MIST (hybrid), and robustness-enhanced variants like Anti-Diffusion (Anti-DF) and MetaCloak. UDAP also generalizes well across SD versions and text prompts, showcasing its practical applicability in real-world scenarios.
\end{abstract}

\begin{links}
    \link{Code}{https://github.com/whulizheng/UDAP}
\end{links}

\section{Introduction}

SD has emerged as a groundbreaking framework in the field of image and video generation~\cite{li2023gligen, ramesh2021zero, gafni2022make, ding2021cogview}, renowned for its ability to synthesize high-quality images and videos. Recent advancements, such as DreamBooth~\cite{ruiz2023dreambooth} and LoRA~\cite{hu2021lora}, have further enhanced the capabilities of SD, positioning it at the forefront of personalized image generation~\cite{chen2023pixartalpha,chen2024pixartdelta,croitoru2023diffusion}. However, all of these methods are vulnerable to adversarial attacks~\cite{van2023anti, li2024pid, liang2023adv}: adding imperceptible noise to the training images can mislead the SD models to generate inaccurate or degraded outputs (shown in the middle panel of Fig.~\ref{fig:teaser}).

In such cases, the models are compromised, resulting in a significant waste of computational resources. Therefore, removing the imperceptible noise in the adversarial images before training the SD models is crucial to defend SD models against such adversarial attacks.

\begin{figure}[!t]
    \centering
    \includegraphics[width=0.85\linewidth]{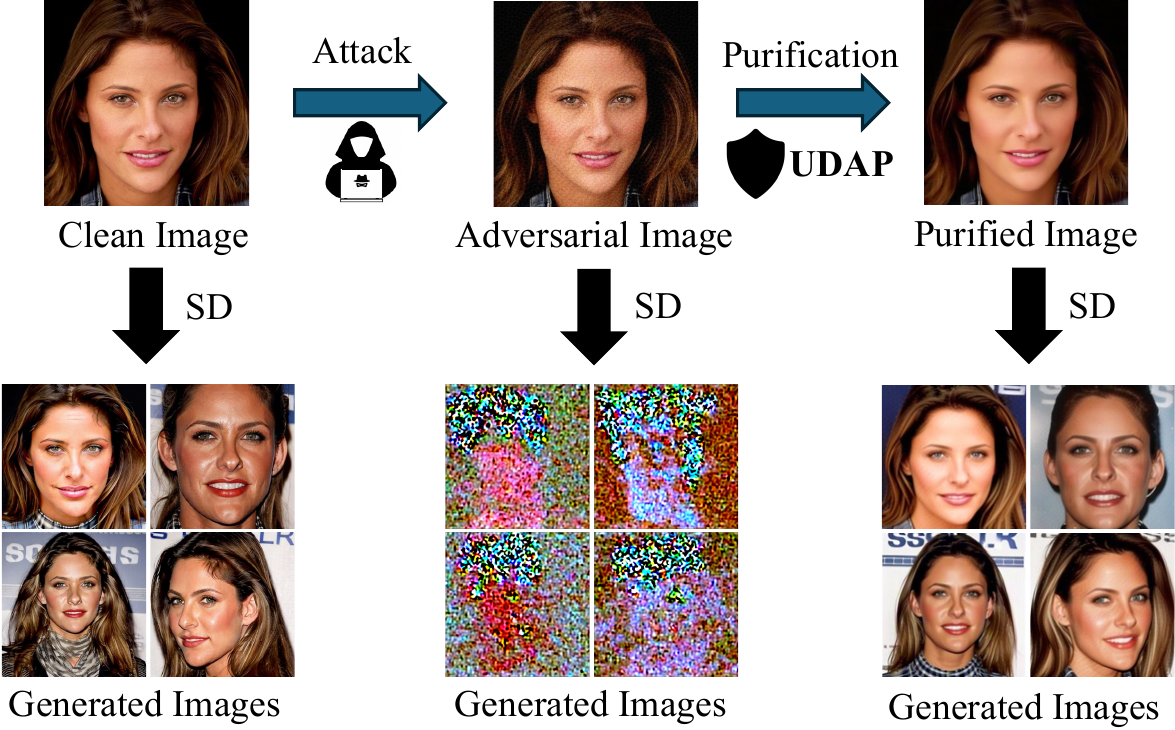}
    \caption{Illustration of the impact of adversarial attacks on SD models and the effectiveness of our proposed UDAP.
    }
    \label{fig:teaser}
\end{figure} 

The majority of work on defending neural networks against adversarial attacks has focused on image classification tasks~\cite{baniecki2024adversarial,li2022defending,chakraborty2021survey} and can be categorized into adversarial training~\cite{bai2021recent,shafahi2020universal,shafahi2019adversarial} and adversarial purification~\cite{costa2024deep,liao2018defense}. In contrast to adversarial training, which is defined to defend against specific attacks that it has been trained on, adversarial purification methods can better defend against previously unseen threats in a plug-and-play manner~\cite{nie2022diffusion}. As a result, significant progress has been made in the area of adversarial purification, evolving from the use of autoencoders in MagNet~\cite{meng2017magnet}, to the application of GANs in Defense-GAN~\cite{samangouei2018defense}, and more recently, the leveraging of diffusion models~\cite{nie2022diffusion,zollicoffer2025lorid,wang2024diffhammer}. While adversarial purification techniques for classification tasks have advanced continuously, to the best of our knowledge, there is no research specifically targeting adversarial purification for SD models. Given the evolving nature of adversarial attack techniques targeting SD, there is an urgent need for adversarial purification methods specifically tailored for SD.

Considering that current adversarial attack methods (e.g., Anti-DreamBooth (Anti-DB)~\cite{van2023anti}, PID~\cite{li2024pid}, and MIST~\cite{liang2023adv}) targeting SD generate adversarial noise are specifically designed to attack different parts of SD~\cite{truong2024attacks}, there is a clear need for a robust method capable of handling various adversarial attacks. Furthermore, in practical scenarios, training sets often contain a mixture of adversarial and clean images. Applying the same adversarial purification to all images would undoubtedly be time-consuming and inefficient, making it impractical for real-world use. Therefore, it is essential to design a dynamic optimization mechanism that can adaptively adjust the strength of purification based on the underlying adversarial noise level in the input images. 

These considerations motivate our proposal of a robust adversarial purification method, specifically tailored to defend against various SD adversarial attacks. Although different SD adversarial attack methods target different parts of the SD model, they share a common characteristic: they are optimized in the latent space of the model. This characteristic implies that the latent corresponding to an adversarial image can alter the results of both the forward and inversion processes in SD. DDIM inversion~\cite{song2020denoising} is a technique that repeatedly performs inversion on the latent before executing the forward process, which therefore can amplify the distance between the reconstructed image and its corresponding adversarial image.
As shown in Fig.~\ref{fig:inv-cmp}, we observe that through DDIM inversion, the $L_{2}$ loss between a clean image and its corresponding reconstruction is small, whereas the $L_{2}$ loss between an adversarial image and its reconstruction is significantly larger.
We propose utilizing the advantageous properties of DDIM inversion as a metric loss. 
Using the latent corresponding to the input image as an initial latent, we perform DDIM inversion on the latent and iteratively optimize it by calculating the $L_{2}$ loss between the reconstructed image and the input image. 
We find that the design of DDIM reconstruction metric loss is simple, yet highly effective, performing well across a variety of adversarial attack methods, SD versions, and diverse prompts.
Moreover, to enhance practical efficiency, we further introduce a dynamic optimization strategy by setting a reconstruction loss threshold as a tradeoff of purification strength and computational cost. This dynamic optimization strategy enables adjustment of purification epochs based on the underlying strength level of adversarial perturbations in the input images. 
As illustrated in Fig.~\ref{fig:teaser}, the proposed UDAP effectively purifies images, enabling the SD model trained on these images to produce highly realistic outputs.

Our key contributions are as follows: \textbf{1)} As far as we know, UDAP is the first universal adversarial purification method specifically designed for SD models.
\textbf{2)} Through theoretical analysis, we demonstrate that adversarial samples targeting SD (regardless of their types) will result in significant DDIM reconstruction errors. This is the theoretical foundation of our universal adversarial purification method.
\textbf{3)} We propose a novel DDIM metric loss to measure the distance between reconstructed and input images. Minimizing this loss optimizes the initial latent into a clean representation, effectively eliminating adversarial noise. Additionally, we introduce a dynamic optimization strategy to adaptively adjust purification epochs, improving efficiency in real-world scenarios. 
\textbf{4)} Through quantitative and qualitative evaluations, our proposed UDAP demonstrates superior purification performance across various adversarial techniques.
For instance, when defending against the PID attack on the CelebA-HQ dataset, our proposed UDAP reduces the Face Detection Failure Rate (FDFR) to ($0.14$), outperforming baselines like GridPure ($0.21$) by over ($33\%$).
Moreover, the dynamic optimization strategy makes UDAP approximately twice faster while maintaining the high purification effectiveness.

\begin{figure}[tbp]
    \centering
    \includegraphics[width=0.8\linewidth]{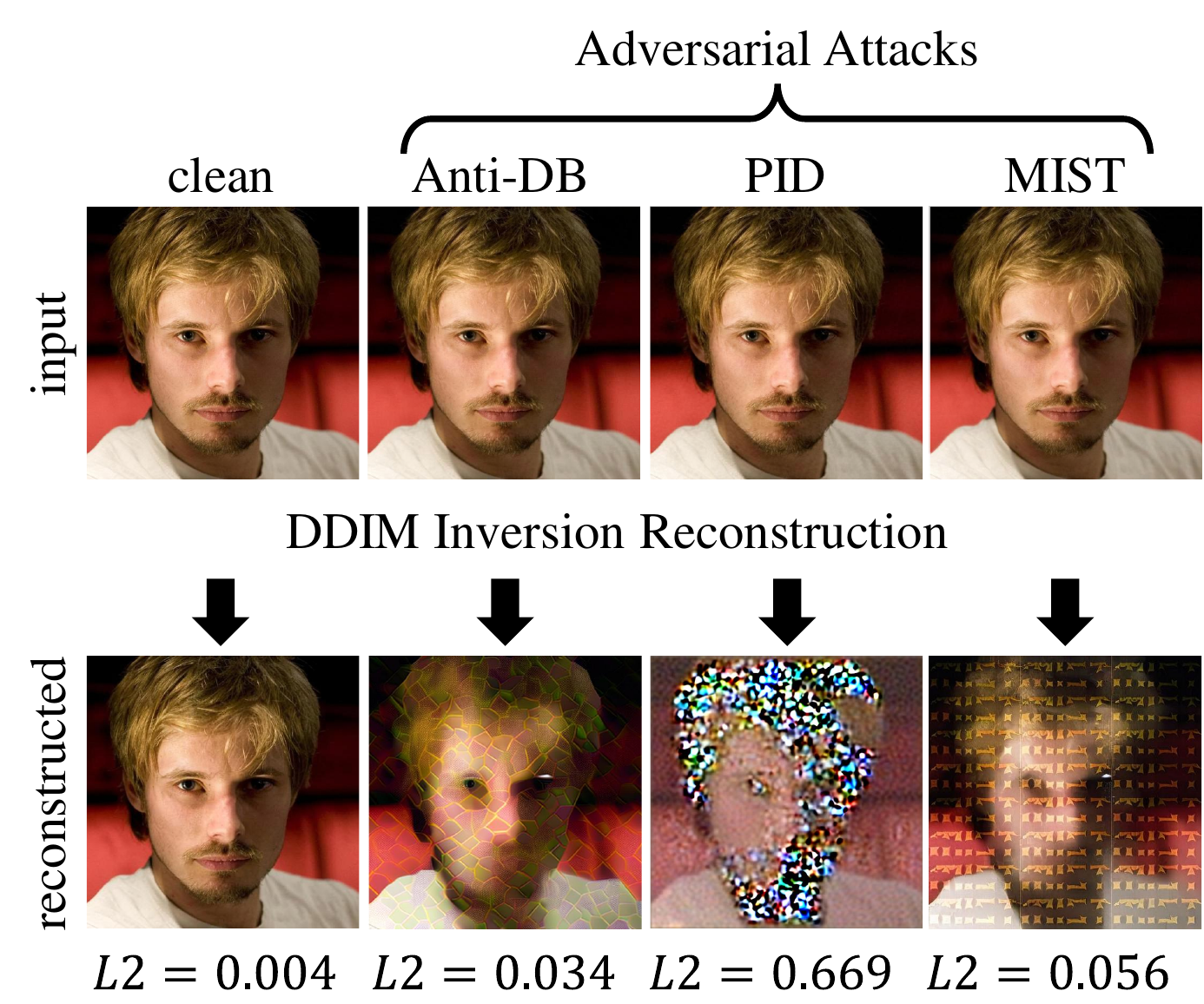}
    \caption{The first row shows the input images and the second row shows their reconstructed images through DDIM inversion reconstruction. Value $L_2$ means the average $L_2$ distance between input and reconstructed images.}
    \label{fig:inv-cmp}
\end{figure}

\section{Related Works}

\subsection{Adversarial Attacks on SD}
\label{sec:related-poi}
Although SD is highly capable of personalized image generation, it is vulnerable to adversarial attacks. These attacks insert malicious examples into the training data to corrupt the model's learning process, resulting in distorted or low-quality image outputs.
MIST creates adversarial noise targeting both the SD model's encoder and its denoising process for a more thorough disruption. It also watermarks adversarial images, which interferes with image generation and makes the attack difficult to mitigate. Anti-DB simulates the DreamBooth fine-tuning process to dynamically attack the model, making standard data purification defenses less effective. In contrast, PID generates prompt-agnostic adversarial noise to corrupt the VAE encoder, disrupting the model's internal image representations for any given text prompt.
Building on these approaches, Anti-DF~\cite{zheng25antidf} enhances attack performance by combining adversarial noise with prompt tuning and injecting semantic disruptions tailored for SD models. Meanwhile, MetaCloak~\cite{liu2024metacloak} introduces a meta-learning framework to generate transferable adversarial perturbations through bi-level optimization, marking the advent of a new era of adaptive adversarial attack techniques.

\subsection{Adversarial Purification}
Adversarial purification has emerged as a key defense against adversarial attacks, complementing adversarial training by removing perturbations from inputs before model processing. Early methods used GANs and energy-based models, but recent advances favor diffusion models for their high-quality reconstructions and robustness against unseen attacks~\cite{yoon2021adversarial}. These diffusion models, particularly SD, leverage iterative denoising to purify adversarial inputs, enhancing classifier robustness~\cite{nie2022diffusion,wang2022guided,lin2024robust,zollicoffer2025lorid,10656194}. 
Despite these advancements of adversarial purification, existing methods primarily focus on classification tasks and do not adequately address the unique challenges posed by adversarial attacks on SD models. These defense paradigms are fundamentally misaligned with the demands of generative models. Their primary objective is to preserve output invariance (i.e., a class label) against input perturbations, whereas the goal for a model like SD is to maintain the intricate perceptual quality and semantic coherence of the entire generated output. 
This core difference in objectives makes the direct application of classification-centric defenses to generative models both ineffective and inappropriate.
Consequently, there is a pressing need for a universal adversarial purification framework specifically designed to counteract diverse adversarial techniques targeting SD. 

\section{Method}
\label{sec:method}
In this section, we first analyze the distinct behaviors of clean and adversarial images under DDIM inversion, which forms the theoretical basis for our method. 
We then present our proposed UDAP framework, followed by a detailed explanation of the inversion optimization using the DDIM metric loss and the dynamic optimization epochs.
\subsection{Analyzing DDIM Inversion Reconstruction}
\label{sec:ddim}
Let $\boldsymbol{x}^{\text{adv}}$ be an adversarial example targeting diffusion models, and let $\hat{\boldsymbol{x}}^{\text{adv}}$ denote its reconstruction via DDIM inversion. 
We introduce a positive constant $Q$ to serve as a lower-bound threshold for the distance, quantifying the minimum impact of a successful adversarial attack.
\begin{proposition}
 $\| \boldsymbol{x}^{\text{adv}} - \hat{\boldsymbol{x}}^{\text{adv}} \|
\ge Q$ when the timestamp of DDIM inversion process approaches the total time steps $T$.
\end{proposition}

\begin{proof}
According to the definition of adversarial examples, for a clean sample $\boldsymbol{x}$, suppose that it exists an adversarial example $\boldsymbol{x}^{\text{adv}} = \boldsymbol{x} + \boldsymbol{\delta}$ with a small perturbation $\left\|\boldsymbol{\delta}\right\|_p \leq \xi$. The noise predictions $\epsilon_{\boldsymbol{\theta}}(\boldsymbol{x}, t, \boldsymbol{c})$ and $\epsilon_{\boldsymbol{\theta}}(\boldsymbol{x}^{\text{adv}}, t, \boldsymbol{c})$ should satisfy such condition:
\begin{equation}
\| \epsilon_{\boldsymbol{\theta}}(\boldsymbol{x}, t, \boldsymbol{c}) - \epsilon_{\boldsymbol{\theta}}(\boldsymbol{x}^{\text{adv}}, t, \boldsymbol{c}) \| \geq Q,
\label{eq:1}
\end{equation}
\noindent where $\epsilon$ is a well trained diffusion model with parameters $\boldsymbol{\theta}$, $\boldsymbol{c}$ is the text embedding of input prompt and $Q$ denotes a clearly perceptible distance.

According to \cite{song2020denoising}, the DDIM inversion process $\boldsymbol{x}_{t} = q_{\boldsymbol{\theta}}(\boldsymbol{x}, t,\boldsymbol{c})$ can be defined as $\boldsymbol{x}_{t} = \sqrt{\bar{\alpha}_t} \boldsymbol{x}_ + \sqrt{1 - \bar{\alpha}_t} \epsilon$, where $\bar{\alpha}_t$ represents the predefined parameters of DDIM.
We can have:
\begin{equation}
\begin{aligned}
\| \boldsymbol{x}_{t} - \boldsymbol{x}_{t}^{\text{adv}} \| & = 
\| \sqrt{\bar{\alpha}_t}(\boldsymbol{x} - \boldsymbol{x}^{\text{adv}}) + \\&  \sqrt{1 - \bar{\alpha}_t}\left(\epsilon_{\boldsymbol{\theta}}\left(\boldsymbol{x}, t, \boldsymbol{c}\right) - \epsilon_{\boldsymbol{\theta}}\left(\boldsymbol{x}^{\text{adv}}, t, \boldsymbol{c}\right)\right) \|
\end{aligned}
\label{eq:2}
\end{equation}
By the definition of adversarial sample $\boldsymbol{x}_{t}^{\text{adv}}$ and substituting \eqref{eq:1} into \eqref{eq:2}, we have $\| \boldsymbol{x}_{t} - \boldsymbol{x}_{t}^{\text{adv}} \|
\geq \| \sqrt{\bar{\alpha}_t} \delta + \sqrt{1 - \bar{\alpha}_t} Q \|.$
When $t \to T$ (total time steps of DDIM), by the definition of DDIM, $\bar{\alpha}_t \to 0$, thus, we can have:
\begin{equation}
\| \boldsymbol{x}_{t} - \boldsymbol{x}_{t}^{\text{adv}} \|
\geq \| \sqrt{\bar{\alpha}_t} \delta + \sqrt{1 - \bar{\alpha}_t} Q \|\approx Q.
\label{eq:3}
\end{equation}
Moreover, by the process of DDIM inversion, both $\boldsymbol{x}_{t}$ and $\boldsymbol{x}_{t}^{\text{adv}}$ should eventually follow Gaussian distributions through iterative noise injection. 

Similarly, we define the DDIM denoise process as $\hat{\boldsymbol{x}}_{0} = p_{\boldsymbol{\theta}}(\boldsymbol{x}_{t}, t,\boldsymbol{c})$. 
Then, by the reversibility of DDIM~\cite{song2020denoising} and given that $\boldsymbol{x}_{t}^{\text{adv}}$ follows a Gaussian distribution, there should exist a normal sample $\boldsymbol{x'}$ such that $\boldsymbol{x}_{t}^{\text{adv}} = q_{\boldsymbol{\theta}}(\boldsymbol{x'}_{0}, t,\boldsymbol{c})$.
By definition, we have $
\hat{\boldsymbol{x}}_{0}^{\text{adv}} = 
p_{\boldsymbol{\theta}}(\boldsymbol{x}_{t}^{\text{adv}}, t,\boldsymbol{c}) = 
p_{\boldsymbol{\theta}}(q_{\boldsymbol{\theta}}(\boldsymbol{x'}_{0}, t,\boldsymbol{c}), t,\boldsymbol{c}) = 
\hat{\boldsymbol{x}}^{'}_{0}
$, where $q_{\boldsymbol{\theta}}(\boldsymbol{x}_{0}, t,\boldsymbol{c})$ and $p_{\boldsymbol{\theta}}(\boldsymbol{x}_{t}, t,\boldsymbol{c})$ are inverse functions of each other and $\hat{\boldsymbol{x}}^{'}_{0}$ is the DDIM inversion reconstructed $\boldsymbol{x}^{'}_{0}$.
Thus, by the reversibility of DDIM, we have: 
\begin{small}
\begin{equation}
\boldsymbol{x'}_{0} \approx \hat{\boldsymbol{x}}^{'}_{0} \approx 
\hat{\boldsymbol{x}}_{0}^{\text{adv}}.
\label{eq:4}
\end{equation}
\end{small}
Under the assumption that $ \epsilon_{\boldsymbol{\theta}} $ is properly trained, it should obey a Lipschitz continuity condition with $ L_t $ (where $L_t$ denotes the Lipschitz constant at timestamp $t$), so we have:
\begin{small}
\begin{equation}
\frac{\| q_{\boldsymbol{\theta}}(\boldsymbol{x}_{0}, t, \boldsymbol{c}) - q_{\boldsymbol{\theta}}(\boldsymbol{x'}_{0}, t, \boldsymbol{c}) \|}{\|\boldsymbol{x}_{0} - \boldsymbol{x'}_{0} \|} = \frac{\| \boldsymbol{x}_{t} - \boldsymbol{x}_{t}^{\text{adv}} \|} {\|\boldsymbol{x}_{0} - \boldsymbol{x'}_{0} \|}  \leq L_t.
\label{eq:5}
\end{equation}
\end{small}
Noting that $\boldsymbol{x}^{\text{adv}}=\boldsymbol{x}+\delta \approx \boldsymbol{x}$ ($\delta$ is a very small perturbation) and substituting \eqref{eq:3} and \eqref{eq:4} into \eqref{eq:5}, we get:
\begin{equation}
\| \boldsymbol{x}^{\text{adv}} - \hat{\boldsymbol{x}}^{\text{adv}} \| \geq \frac{Q}{L_t}.
\label{eq:6}
\end{equation}
As $t \to T$, according to the Lipschitz continuity of diffusion models, we have $L_t \ll 1$~\cite{yang2024lipschitz}. So \eqref{eq:6} can be rewritten as:
\begin{equation}
\| \boldsymbol{x}^{\text{adv}} - \hat{\boldsymbol{x}}^{\text{adv}} \|
\geq \frac{Q}{L_t}
\gg Q.
\label{eq:7}
\end{equation}
Therefore, the distance between $\boldsymbol{x}^{\text{adv}}$ and $\hat{\boldsymbol{x}}^{\text{adv}}$ is much larger than $Q$, where $Q$ denotes a clearly perceptible distance by the definition of adversarial examples.
\end{proof}

\begin{figure*}[ht]
    \centering
    \includegraphics[width=1\linewidth]{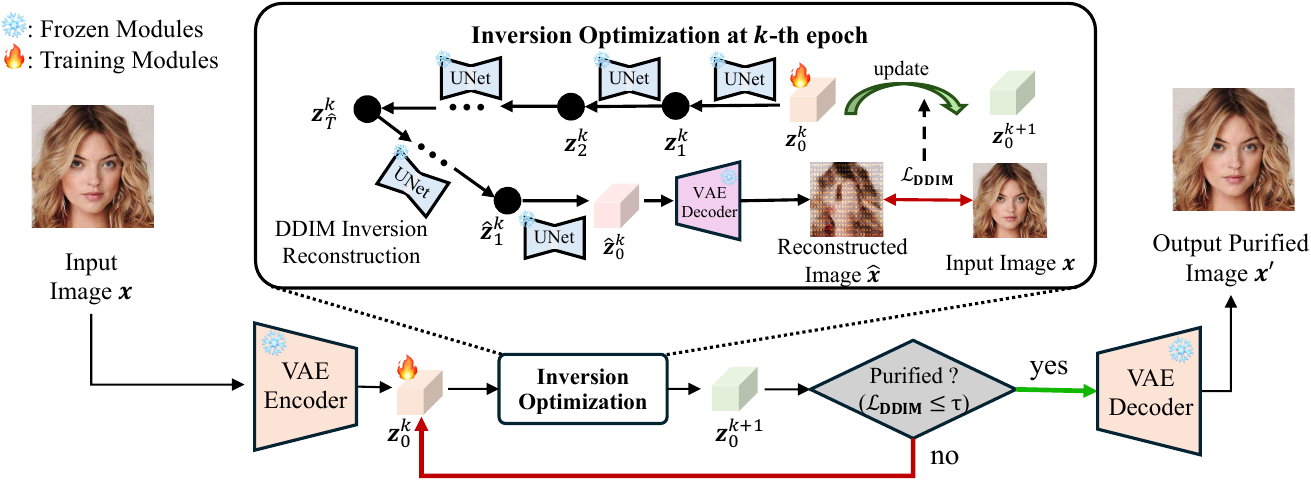}
    \caption{The overall framework of UDAP. 
    The input image is first encoded into a latent space representation, where it undergoes iterative purification through inversion optimization. This optimization is guided by our proposed DDIM metric loss, $\mathcal{L}_{\mathrm{DDIM}}$. Once the latent representation has DDIM metric loss that is less than $\tau$ (deemed as sufficiently purified), it is decoded back into the image space using a VAE decoder, resulting in the final purified image.}
    \label{fig:frame}
\end{figure*} 

Whether an attack targets the UNet or the VAE, it ultimately introduces errors in the latent space that propagate to the final image reconstruction during DDIM inversion.
As can be seen from Fig.~\ref{fig:inv-cmp}, the experimental results further verify this conclusion. 
As demonstrated in this figure, while clean images maintain accurate reconstruction quality (with low $L_2$ distance $0.004$), all adversarial images—regardless of their attack types (Anti-DB for UNet-targeted attacks, PID for VAE-targeted attacks, or MIST for mixed attacks)—show significant differences (with high $L_2$ distances $0.034$, $0.669$ and $0.056$) between the reconstructed images and the original input.
These observations highlight that adversarial images can be distinguished from clean images by DDIM inversion reconstruction. 
As a result, we integrate DDIM inversion reconstruction into our optimization process which forms the foundation of the proposed UDAP.

\subsection{Overview Framework}
The overall framework of UDAP is illustrated in Fig.~\ref{fig:frame}. The process begins with an input image $\boldsymbol{x}$, which is first encoded by the VAE encoder of SD to obtain the initial latent representation $ \boldsymbol{z}^0_0=\mathcal{E}\left( \boldsymbol{x}\right)$. 
At the $k$-th epoch, the latent $\boldsymbol{z}^k_0$ undergoes inversion optimization to produce $ \boldsymbol{z}^{k+1}_0$. 
The value of $\mathcal{L}_\mathrm{DDIM}$ at each epoch is also used to determine whether $\boldsymbol{z}^{k+1}_0$ has been sufficiently purified. This is achieved by comparing $\mathcal{L}_\mathrm{DDIM}$ with a predefined threshold $\tau$, which is a tradeoff of purification strength and computational cost. If $\mathcal{L}_\mathrm{DDIM}>\tau$, indicating insufficient purification, $\boldsymbol{z}^{k+1}_0$ is sent to the next epoch for further inversion optimization. Otherwise, if $\mathcal{L}_\mathrm{DDIM}\le\tau$, the latent is considered purified, and $\boldsymbol{z}^{k+1}_0$	is decoded by the pretrained VAE decoder to produce the final purified image $\boldsymbol{x}'$.

\subsection{Inversion Optimization}
As analyzed above, the DDIM inversion reconstruction error serves as a reliable metric for distinguishing between clean and adversarial images. Building on this insight, we propose the DDIM metric loss $\mathcal{L}_{\mathrm{DDIM}}$ as an indicator of adversarial images by the DDIM reconstruction. By minimizing the DDIM metric loss $\mathcal{L}_{\mathrm{DDIM}}$, inversion optimization process can remove adversarial noise from the latent representation $\boldsymbol{z}^{k}_0$. 
Specifically, leveraging a pretrained UNet $\epsilon_{\boldsymbol{\theta}}$, the latent $\boldsymbol{z}^{k}_0$ is first sampled to timestamp $\hat{T}$ through the DDIM inversion process $q_{\boldsymbol{\theta}}$ as:
\begin{equation}
q_{\boldsymbol{\theta}}\left(\boldsymbol{z}^k_{1:\hat{T}}|\boldsymbol{z}^k_0\right)=\prod \limits_{t=1}^{\hat{T}}q_{\boldsymbol{\theta}}(\boldsymbol{z}^k_t|\boldsymbol{z}^k_{t-1}).
\end{equation}

The inverted latent $z^k_{\hat{T}}$ is then sampled back to timestamp $0$ by the DDIM process $p_{\boldsymbol{\theta}}$ as:
\begin{equation}
p_{\boldsymbol{\theta}}(\hat{\boldsymbol{z}}^k_{\hat{T}-1:0}|\hat{\boldsymbol{z}}^k_{\hat{T}})=\prod \limits_{t=1}^{\hat{T}}p_{\boldsymbol{\theta}}(\hat{\boldsymbol{z}}^k_{t-1}|\hat{\boldsymbol{z}}^k_{t}),
\end{equation}

\noindent where $\hat{\boldsymbol{z}}^k_{\hat{T}}=\boldsymbol{z}_{\hat{T}}^k$ as an initialization. 
The reconstructed image $\hat{\boldsymbol{x}}$ is obtained by decoding $\hat{\boldsymbol{z}}^k_0$ with the pretrained VAE decoder as $\hat{\boldsymbol{x}}=\mathcal{D}(\hat{\boldsymbol{z}}^k_0)$. The DDIM metric loss, which quantifies the distance between the reconstructed image $\hat{\boldsymbol{x}}$ and the input image $\boldsymbol{x}$, is computed using the loss function $\mathcal{L}_{\mathrm{DDIM}}$:

\begin{small}
\begin{equation}
    \mathcal{L}_{\mathrm{DDIM}} =\left\|\mathcal{D}\left(p_{\boldsymbol{\theta}}\left(\hat{\boldsymbol{z}}^k_{\hat{T}-1:0}|q_{\boldsymbol{\theta}}\left(\boldsymbol{z}^k_{1:\hat{T}}|\boldsymbol{z}^k_0\right)\right)\right)-\boldsymbol{x}\right\|_2^2.
\end{equation}
\end{small}

Thus, the inversion optimization problem $P.1$ of $z^k_0$ can be formally defined as:
\begin{small}
\begin{equation}
   P.1\quad\min\limits_{\boldsymbol{z}^k_0}~\mathbb{E}_{\epsilon_{\boldsymbol{\theta}},\boldsymbol{z}^0_0=\mathcal{E}\left(\boldsymbol{x}\right)}\left[\mathcal{L}_{\mathrm{DDIM}}\left(\boldsymbol{x},\boldsymbol{z}^k_0,\boldsymbol{c},\hat{T}\right)\right].
\end{equation}
\end{small}
In summary, this optimization process iteratively refines the latent representation $\boldsymbol{z}^k_0$ to minimize the DDIM metric loss, effectively removing adversarial noise while preserving the content of the image.

\begin{small}

\begin{table*}[ht]
\centering
\begin{tabular}{c|ccccccc}
\hline
Attacks                   & Purification & FDFR↓         & ISM↑          & SER-FQA↑      & BRISQUE↓       & FID↓            & NIQE↓         \\ \hline
\multirow{4}{*}{PID}        & UDAP         & \textbf{0.14} & \textbf{0.57} & \textbf{0.59} & \textbf{21.45} & \textbf{185.72} & \textbf{4.38} \\
                            & DiffPure     & 0.24          & 0.42          & 0.42          & 29.93          & 238.52          & 4.54          \\
                            & GridPure    & 0.21          & 0.48          & 0.46          & 25.29          & 205.62          & 4.46          \\ \cline{2-8} 
                            & -            & 0.87          & 0.02          & 0.06          & 48.24          & 414.35          & 5.96          \\ \hline
\multirow{4}{*}{MIST}       & UDAP         & 0.11          & \textbf{0.61} & \textbf{0.73} & \textbf{20.42} & \textbf{163.25} & \textbf{4.24} \\
                            & DiffPure     & \textbf{0.10} & 0.56          & 0.67          & 25.64          & 224.33          & 4.53          \\
                            & GridPure    & 0.11          & 0.58          & 0.66          & 22.83          & 185.72          & 4.35          \\ \cline{2-8} 
                            & -            & 0.12          & 0.51          & 0.63          & 33.72          & 273.37          & 4.93          \\ \hline
\multirow{4}{*}{Anti-DB}    & UDAP         & \textbf{0.09} & \textbf{0.62} & \textbf{0.72} & \textbf{17.53} & \textbf{142.54} & \textbf{4.21} \\
                            & DiffPure     & 0.14          & 0.58          & 0.66          & 26.48          & 205.82          & 4.48          \\
                            & GridPure    & 0.11          & 0.58          & 0.69          & 22.75          & 158.54          & 4.74          \\ \cline{2-8} 
                            & -            & 0.55          & 0.40           & 0.38          & 38.24          & 342.36          & 5.58          \\ \hline
\multirow{4}{*}{Anti-DF} & UDAP         & \textbf{0.22} & \textbf{0.51} & \textbf{0.64} & \textbf{25.08} & \textbf{202.67} & \textbf{4.56} \\
                            & DiffPure     & 0.46          & 0.33          & 0.55          & 29.50          & 356.22          & 4.57          \\
                            & GridPure    & 0.33          & 0.46          & 0.56          & 27.32          & 235.11          & 4.62          \\ \cline{2-8} 
                            & -            & 0.59          & 0.25         & 0.45          & 39.46          & 257.98          & 5.82          \\ \hline
\multirow{4}{*}{MetaCloak} & UDAP         & \textbf{0.24} & \textbf{0.53} & \textbf{0.58} & \textbf{30.36} & \textbf{264.81} & \textbf{4.49} \\
                            & DiffPure     & 0.54          & 0.31          & 0.44          & 35.15          & 332.52          & 4.62          \\
                            & GridPure    & 0.48          & 0.35          & 0.51          & 34.52          & 325.74          & 4.64          \\ \cline{2-8} 
                            & -            & 0.73          & 0.03          & 0.08          & 47.38          & 404.82          & 5.88          \\ \hline
clean                       & -            & 0.09          & 0.63          & 0.74          & 18.36          & 142.38          & 4.34          \\ \hline
\end{tabular}
\caption{Comparison on the purification performance of different methods on the DreamBooth model on dataset CelebA-HQ. The best-performing purification under each metric is marked with \textbf{bold}.}
\label{tab:dreambooth_cmp}

\end{table*}
\end{small}

\subsection{Dynamic Optimization Epochs}
In practical scenarios, datasets often contain a mixture of adversarial and clean images. Applying the same number of optimization epochs to all images would be computationally inefficient, especially when many images are already clean or require minimal purification. To address this, we introduce a dynamic optimization epochs strategy, which adaptively adjusts the number of optimization epochs based on the DDIM metric loss of each image.
The key idea is to set a tradeoff threshold $\tau$  that determines when the optimization process should terminate. If the DDIM metric loss $\mathcal{L}_{\mathrm{DDIM}}$ for a given latent $\boldsymbol{z}^{k+1}_0$ exceeds $\tau$, indicating that the image is not yet sufficiently purified, the latent will be sent to the next epoch for further optimization. The optimization continues until either $\mathcal{L}_{\mathrm{DDIM}}\le\tau$ or the maximum epoch $K$ is reached.
This threshold should be able to represent the average performance of clean images in DDIM inversion reconstruction.
Therefore, we set the threshold $\tau$ as the average DDIM metric loss on a batch of $N$ clean images, which can be estimated as:
\begin{equation}
    \tau \approx \frac{1}{N} \sum_{n=1}^{N} \mathcal{L}_{\text{DDIM}}(\boldsymbol{x}_n, \epsilon_{\boldsymbol{\theta}}, c, \hat{T}).
\label{eq-tqu}
\end{equation}
In our UDAP, $\tau$ is estimated by 1000 clean images from ImageNet~\cite{deng2009imagenet} as  $\tau=4\times10^{-3}$. 
The impact of different $\tau$ values on the purification performances and efficiency of UDAP will also be discussed in subsequent ablation studies.
By dynamically adjusting the number of optimization epochs based on the DDIM metric loss, UDAP achieves a balance between purification quality and computational efficiency. This strategy ensures that heavily adversarial perturbed images undergo sufficient optimization, while clean or minimally perturbed images are processed efficiently, making the framework highly practical for real-world applications.

\begin{figure*}[htbp]
    \centering
    \includegraphics[width=0.9\linewidth]{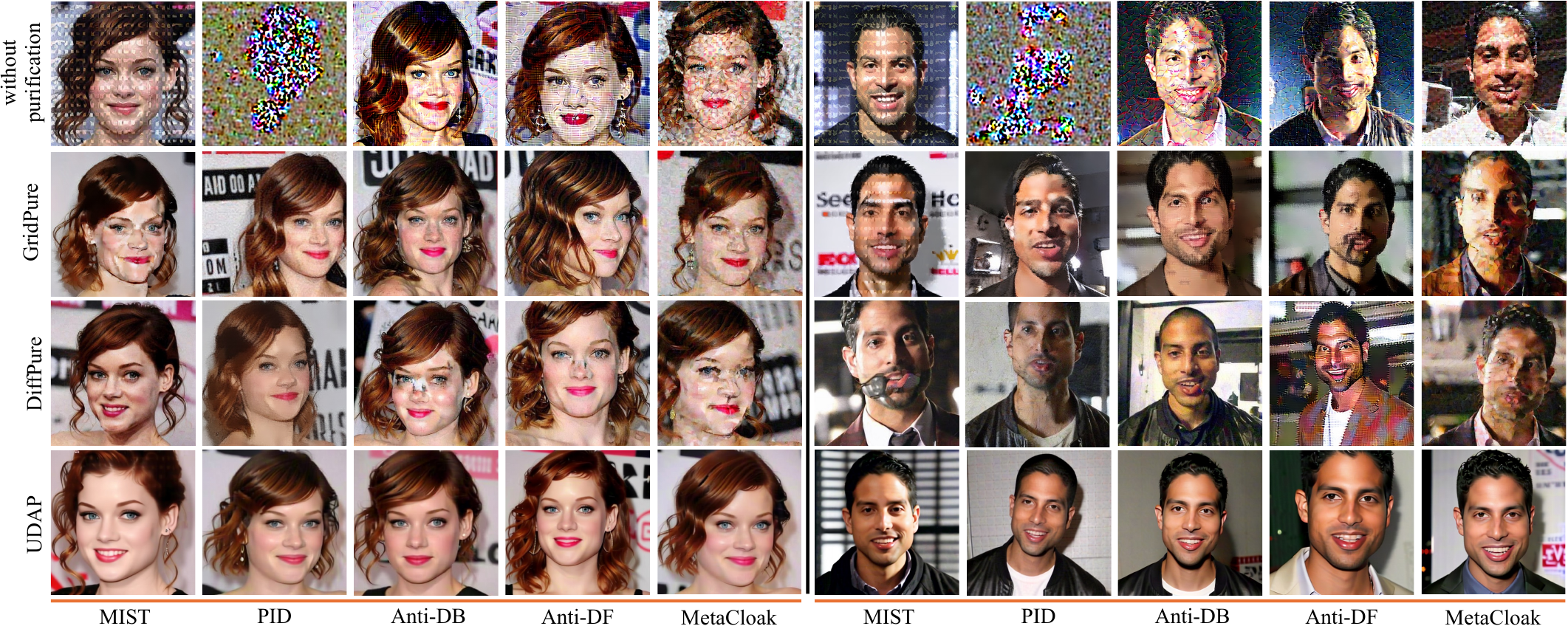}
    \caption{Qualitative purification results of different methods against different adversarial attacks on the DreamBooth model. The specific prompt adopted is “a photo of sks person”. The instances are from CelebA-HQ (left) and VGGFace2 (right).}
    \label{fig:cmp_db}
\end{figure*} 

\section{Experiment}
\subsection{Implementation Details}

\textbf{Datasets.} To conduct adversarial attacks and purifications and to train the DreamBooth models, we utilize the dataset methodology from Anti-DB. Specifically, we conduct experiments using 100 unique identifiers (IDs) sourced from the VGGFace2~\cite{cao2018vggface2} and CelebA-HQ~\cite{karras2017progressive} datasets where each ID contains 4 images.

\noindent\textbf{Training Configurations.}
During the purification process of UDAP, the threshold $\tau$ is $4 \times10^{-3}$ and the max epoch number $K$ is $100$ by default. 
To balance the consumption of GPU memory and the precision of DDIM Inversion, we set the total inference steps $T$ and max depth $\hat{T}$ for DDIM inversion to $20$ and $10$, respectively. Null prompt is used in the DDIM inversion.
The entire purification process, when executed on 8 NVIDIA A6000 GPUs, takes approximately $20$ minutes for $400$ images with the shape of $512 \times 512$  (around $3$ seconds per image). 

\noindent\textbf{Evaluation Metrics.} We perform purification on adversarial images using various adversarial attack methods. Subsequently, we fine-tune SD with DreamBooth to evaluate the purification performance. To measure the purification performance on the DreamBooth models, following Anti-DB, we adopt the following four metrics: BRISQUE~\cite{mittal2012no}, SER-FQA~\cite{terhorst2020ser}, FDFR~\cite{deng2020retinaface}, and ISM~\cite{deng2019arcface}. 
Additionally, we introduce two more Image Quality Assessment (IQA) metrics: Fréchet Inception Distance (FID)~\cite{heusel2017gans} and Natural Image Quality Evaluator (NIQE)~\cite{mittal2012making}.

\subsection{Comparison with Purification Baselines}
We evaluate UDAP's performance against a diverse set of adversarial attacks, including VAE-targeted (PID), UNet-targeted (Anti-DB), and mixed-strategy (MIST) methods. We also tested our UDAP on robust adversarial attack methods (Anti-DF and MetaCloak) that are specially designed to attack purification methods. For baselines, though most adversarial purification methods are targeting classification, we also test recent diffusion based general purification methods, like DiffPure~\cite{nie2022diffusion} and GridPure\cite{10656194}, as a comparison. All these methods are processed in their default configurations.
During the evaluation process, for each trained DreamBooth model, we generate $16$ images under $5$ different seeds, totaling $80$ images, to evaluate the corresponding results, thereby eliminating the variability associated with a single seed.

The quantitative results on DreamBooth for purification of different adversarial methods are shown in Tab.~\ref{tab:dreambooth_cmp}. Compared to clean images, adversarial attacks severely disrupt DreamBooth's performance; for instance, the MetaCloak causes the identity-preserving ISM score to plummet from 0.63 to 0.03. After applying various purification methods, performance improves, but UDAP consistently achieves the best results across most metrics. Against the Anti-DB, UDAP-purified images produce generations with superior facial integrity, achieving the lowest FDFR (0.09) and highest SER-FQA (0.72). It also best recovers the image’s ID features, reaching the highest ISM value of 0.57 against PID, a significant improvement over baselines like DiffPure (0.42). Additionally, UDAP delivers the best image quality, achieving the lowest FID score of 264.81 against the robust MetaCloak attack, whereas other methods score above 325.
In summary, for images from CelebA-HQ, UDAP provides superior adversarial purification performance. More results on VGG-Face can be found in the supplementary.

The qualitative results in Fig.~\ref{fig:cmp_db} further support that UDAP provides superior adversarial purification performance. While methods like GridPure and DiffPure offer some level of purification by promoting the visual quality of generated images, UDAP could purify most adversarial perturbations and make DreamBooth to generate images with the best visual quality. The advantages of UDAP are very evident, especially in purifying images affected by robust methods such as Anti-DF and MetaCloak, where it can minimize visual loss to the greatest extent.
As shown in Fig.~\ref{fig:cmp_pu}, UDAP can remove adversarial noise from adversarial images for both VAE targeted method (PID) and UNet targeted method (Anti-DB), while other purification method still leave a lot of visible noise on the purified images.

\begin{figure}[t]
    \centering
    \includegraphics[width=0.95\linewidth]{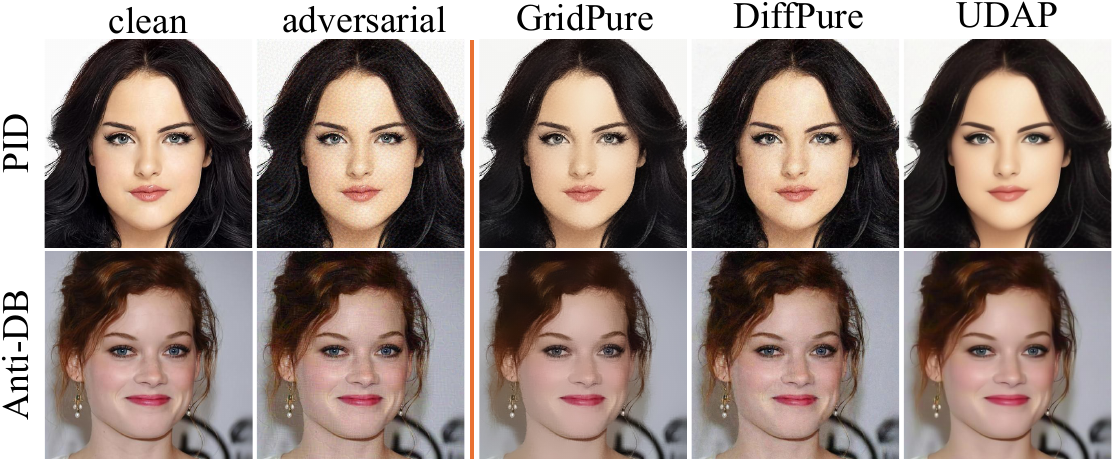}
    \caption{Comparison on the adversarial images and purified images under different adversarial attack and purification methods.}
    \label{fig:cmp_pu}
\end{figure}

\subsection{Ablation Study}
To dynamically adjust the number of inversion optimization epochs, a preset threshold $\tau$ is used to determine whether the latent $z_0$ images are well purified. 
To analyze the impact of $\tau$ on the purification performance and efficiency of UDAP, we conduct comparative experiments with different values of $\tau$ on DreamBooth against Anti-DB.
We constructed a dataset based on CelebA-HQ consisting of 50 unique IDs, with 4 images per ID. 
To simulate real-world scenarios where clean and adversarial images are mixed, 2 images per ID are adversarial perturbed using Anti-DB, while the remaining 2 images are kept clean. 
The experimental results are presented in Tab.~\ref{tab:time}.The first four rows show the purification performance for $\tau$ ranging from $2\times10^{-3}$ to $5\times10^{-3}$. As $\tau$ increases, both the purification performance and the average processing time decrease. Notably, $\tau=4\times10^{-3}$ strikes an optimal balance, offering strong purification performance while maintaining a relatively short processing time.

For comparison, the last two rows in Tab.~\ref{tab:time} show the purification performance when using fixed numbers of epochs ($100$ and $50$). While reducing the number of epochs decreases the average processing time, it also leads to a significant drop in purification performance. These results highlight the effectiveness of UDAP's dynamic optimization strategy, which adaptively adjusts the number of epochs based on the DDIM metric loss, ensuring both efficiency and high-quality purification.
\begin{small}
\begin{table}[ht]
\centering
\begin{tabular}{c|ccc|c}
\hline
value of $\tau$          & ISM↑ & BRISQUE↓ & FID↓   & time(s) \\ \hline

$2\times10^{-3}$          & 0.66 & 18.33    & 146.72 & 18               \\
$3\times10^{-3}$          & 0.66 & 18.68    & 144.81 & 7               \\
$4\times10^{-3}$          & 0.63 & 19.82    & 145.66 & 3                \\
$5\times10^{-3}$          & 0.28 & 24.54    & 232.91 & 2                \\ \hline
epochs=$100$   & 0.68 & 18.13    & 144.27 & 15               \\
epochs=$50$    & 0.22 & 21.10    & 164.53 & 8                \\ \hline
\end{tabular}
\caption{Comparison of purification performance on DreamBooth against Anti-DB with different values of $\tau$. The last two rows represent results with fixed epochs (100 and 50). The "time" column indicates the average time cost per image. The dataset is derived from CelebA-HQ, where half of the images are adversarial, and the remaining half are clean.}
\label{tab:time}
\end{table}
\end{small}

To further assess the impact of UDAP on clean images, we conduct a comparative evaluation between ``adv.'' and ``clean'' datasets using DreamBooth against Anti-DB. The datasets are constructed from CelebA-HQ, with each identity represented by four images. Here, ``adv.'' refers to datasets composed entirely of adversarial examples, while "clean" denotes datasets containing only clean images. As shown in Tab.~\ref{tab:clean}, the experimental results demonstrate negligible performance differences between clean images processed with and without UDAP. This indicates that UDAP has minimal effect on the behavior of clean images.
\begin{small}
\begin{table}[ht]
\centering
\begin{tabular}{c|c|ccc}
\hline
UDAP &dataset         & ISM↑ & BRISQUE↓ & FID↓   \\ \hline

\checkmark      &adv.   & 0.62 & 17.53    & 142.54                \\
\checkmark     &clean     & 0.62 & 18.41    & 144.47               \\  \hline
$\times$ &adv. & 0.40              & 38.24          & 342.36  \\

$\times$ &clean & 0.63              & 18.36          & 142.38  \\

      \hline

\end{tabular}
\caption{Comparison on the purification performance of UDAP on DreamBooth against Anti-DB under datasets with different degrees of adversarial perturbations.}
\label{tab:clean}
\end{table}
\end{small}

\subsection{Unexpected Scenarios}

In practical scenarios, the specific utilization of the SD models by malicious users to perform adversarial attacks is unpredictable.
We perform experiments to evaluate the purification performance of UDAP across different SD versions when they are match or mismatch. Specifically, we attacks the CelebA-HQ dataset using Anti-DB based on SD v$1.4$ and v$2.1$, respectively. We then purified the adversarial images using UDAP based on different SD versions. Subsequently, we trained SD using purified images with the DreamBooth method and analyzed the quality of the generated images. The experimental results are shown in Tab.~\ref{tab:ver}, which demonstrates that UDAP can effectively purify Anti-DB across SD versions, even when the SD versions used for adversarial attacks and purification do not match.
\begin{small}

\begin{table}[ht]
\begin{tabular}{c|c|cccc}
\hline
Adv.                  & Pur. & FDFR↓ & ISM↑ & BRISQUE↓ & FID↓   \\ \hline
\multirow{3}{*}{v2.1} & v2.1 & 0.09  & 0.62 & 17.53    & 142.54 \\
                      & v1.4 & 0.10  & 0.61 & 18.73    & 144.78 \\
                      & no   & 0.55  & 0.40  & 38.24    & 342.36 \\ \hline
\multirow{3}{*}{v1.4} & v2.1 & 0.09  & 0.63 & 17.24    & 146.79 \\
                      & v1.4 & 0.11  & 0.64 & 16.37    & 141.29 \\
                      & no   & 0.54  & 0.38 & 38.36    & 332.76 \\ \hline
\end{tabular}
\caption{Comparison on the purification performance of UDAP against Anti-DB under different versions of SD on dataset CelebA-HQ. The terms “Adv.” and “Pur.” refer to the SD version for adversarial attacks with Anti-DB and purifying with UDAP.}
\label{tab:ver}
\end{table}
\end{small}

For DreamBooth, different prompts can be used to generate various content. To investigate the impact of different prompts on the purification effect of UDAP against Anti-DB, we additionally introduce three prompts: p1 , p2, and p3, which are “a photo of sks person with sad face”, “facial close-up of sks person” and “a photo of sks person yawning in speech” respectively, to evaluate the purification performance. We can see from Tab.~\ref{tab:p} that UDAP can also provide purification across different prompts. 

\begin{table}[ht]
\centering
\begin{tabular}{c|c|cccc}
\hline
P                   & Pur. & FDFR↓         & ISM↑          & BRISQUE↓       & FID↓            \\ \hline
\multirow{2}{*}{p1} & yes  & \textbf{0.11} & \textbf{0.53} & \textbf{17.28} & \textbf{183.39} \\
                    & no   & 0.44          & 0.32          & 37.52          & 346.2           \\ \hline
\multirow{2}{*}{p2} & yes  & \textbf{0.07} & \textbf{0.44} & \textbf{16.83} & \textbf{153.82} \\
                    & no   & 0.62          & 0.13          & 29.51          & 322.71          \\ \hline
\multirow{2}{*}{p3} & yes  & \textbf{0.09} & \textbf{0.33} & \textbf{19.72} & \textbf{201.44} \\
                    & no   & 0.67          & 0.11          & 36.77         & 412.52          \\ \hline
\end{tabular}
\caption{Comparison on the purification performance with different inference prompts on dataset CelebA-HQ. “P” and “Pur.” refer to prompt and purification.}
\label{tab:p}
\end{table}

\section{Conclusion}
This paper presents UDAP, a universal adversarial purification framework specifically designed to mitigate adversarial attacks on SD. UDAP leverages the distinct reconstruction behaviors of clean and adversarial images during DDIM inversion, introducing a DDIM metric loss that effectively eliminates adversarial perturbations while preserving the core content of the input image. UDAP also incorporates a dynamic epoch adjustment strategy, which adaptively optimizes the purification process and significantly improves its computational efficiency. Extensive experimental results demonstrate that UDAP surpasses existing methods in defending against diverse adversarial techniques. Furthermore, UDAP achieves superior purification under cross-version SD scenarios and varying inference prompts, showcasing its generalizability in real-world applications.

\section{Acknowledgments}
This work was supported in part by Macau Science and Technology Development Fund under 001/2024/SKL, 0119/2024/RIB2, and 0022/2022/A1; in part by Research Committee at University of Macau under MYRG-GRG2023-00058-FST-UMDF; in part by the Guangdong Basic and Applied Basic Research Foundation under Grant 2024A1515012536.

\bibliography{aaai2026}

\end{document}